\documentclass{new_tlp}
\usepackage{mathptmx}

  \title%[???]
        {Enhancing magic sets with an application to ontological reasoning}

  \author[M. Alviano et al.]
         {MARIO ALVIANO, NICOLA LEONE, PIERFRANCESCO VELTRI, JESSICA ZANGARI\\
         Department of Mathematics and Computer Science, University of Calabria, Italy \\
         \email{\{alviano,leone,veltri,zangari\}@mat.unical.it}}

\usepackage[ruled,linesnumbered,vlined]{algorithm2e}
\usepackage{amsmath}
\usepackage{amssymb}
\usepackage{tikz}\usetikzlibrary{arrows,fit,shapes.symbols,patterns}
\usepackage{pgfplots}\usetikzlibrary{plotmarks}
\usepackage{xspace}
\usepackage{url}

\usepackage{thm-restate}

\usepackage{listings}
\makeatletter
\lst@Key{countblanklines}{true}[t]{\lstKV@SetIf{#1}\lst@ifcountblanklines}
\lst@AddToHook{OnEmptyLine}{%
    \lst@ifnumberblanklines\else%
    \lst@ifcountblanklines\else%
    \advance\c@lstnumber-\@ne\relax%
    \fi%
    \fi}
\makeatother
\lstdefinelanguage{asp}{
    morekeywords={not, count, sum, min, max},
    morecomment=[l]{\%},
    breakatwhitespace=true,
    captionpos=b,
    numbers=left,
    numbersep=5pt,
    numberblanklines=false,
    countblanklines=false,
    commentstyle=\color{gray},
    frame=bt, framexbottommargin=5pt, framextopmargin=5pt,
    aboveskip=5pt, belowskip=5pt,
    abovecaptionskip=10pt
}
\lstnewenvironment{asp}{
	\lstset{
		language=asp,
		showstringspaces=false,
        keepspaces=true,
		formfeed=\newpage,
		tabsize=4,
		commentstyle=\color{blue},
		numbers=none,
		breaklines=true,
		literate={~} {$\sim$}{1},
    frame=none,
	}
}{}
\newtheorem{definition}{Definition}[section]
\newtheorem{example}{Example}[section]

\def\la{\ensuremath{\textnormal{ :-- }}}
\newcommand{\tuple}[1]{\ensuremath{\left \langle #1 \right \rangle }}
\newcommand{\magicRules}[1]{\ensuremath{R^{\it mgc}}\xspace}
\newcommand{\modifiedRules}[1]{\ensuremath{R^{\it mod}}\xspace}

\begin{document}

\label{firstpage}

\maketitle

\begin{abstract}
Magic sets are a Datalog to Datalog rewriting technique to optimize query answering.
The rewritten program focuses on a portion of the stable model(s) of the input program which is sufficient to answer the given query.
However, the rewriting may introduce new recursive definitions, which can involve even negation and aggregations, and may slow down program evaluation.
This paper  enhances the magic set technique by preventing the creation of (new) recursive definitions in the rewritten program. It turns out that the new version of magic sets is closed for Datalog programs with stratified negation and aggregations, which is very convenient to obtain efficient computation of the stable model of the rewritten program. Moreover, the rewritten program is further optimized by the elimination of subsumed rules and by the efficient handling of the cases where binding propagation is lost.
The research was stimulated by a challenge on the exploitation of Datalog/\textsc{dlv} for efficient reasoning on large ontologies. All proposed techniques have been hence implemented in the \textsc{dlv} system, and tested for ontological reasoning,  confirming their effectiveness.

\medskip
\noindent
\emph{Under consideration for publication in Theory and Practice of Logic Programming.}
\end{abstract}


\section{Introduction}

Datalog is a rule based language for knowledge representation and reasoning suitable for a natural declaration of inductive definitions and ontological reasoning \cite{DBLP:conf/aaai/EiterOSTX12}.
Several extensions to the core language of Datalog exist, among them default negation \cite{DBLP:journals/jlp/Gelder89,DBLP:journals/jacm/GelderRS91,DBLP:journals/ngc/GelfondL91} and aggregates \cite{DBLP:journals/ai/SimonsNS02,DBLP:journals/tplp/PelovDB07,DBLP:journals/ai/LiuPST10,DBLP:conf/aaaiss/BartholomewLM11,DBLP:journals/tocl/Ferraris11,DBLP:journals/tplp/GelfondZ14}.
Restrictions on the use of these linguistic constructs lead to preserve the existence and uniqueness of the stable model associated with a knowledge base;
specifically, such restrictions essentially enforce a stratification on the definitions involving negation and aggregates \cite{DBLP:journals/ai/FaberPL11}.
The semantics of the resulting language reached a broad consensus in the knowledge representation and reasoning community, as in fact the notions of \emph{perfect model}, \emph{well-founded model}, and \emph{stable model} coincide for stratified programs \cite{DBLP:journals/jar/Przymusinski89,DBLP:journals/jacm/GelderRS91}.

The stable model of a Datalog program can be constructed bottom-up, starting from facts in the program, and deriving new atoms from rules whose bodies become true.
Negation and aggregates are handled by partitioning the input program into different strata, so that the lowest stratum does not contain negation and aggregates, and each other stratum only negates and aggregates over predicates of lower strata.
Such a bottom-up procedure is very efficient for producing the stable model, but it may be by itself inefficient for query answering.
In fact, the stable model may contain atoms that are not relevant to answer the given query, and therefore constitute a source of inefficiency for query answering.
In contrast, top-down procedures start from the query, and consider bodies of the rules defining the query predicate as subqueries.
Hence, the computation focuses on a portion of the stable model that is relevant to answer the query.

The magic sets algorithm is a top-down rewriting of the input program that restricts the range of the object variables so that only the portion of the stable model that is relevant to answer the query is materialized by a bottom-up evaluation of the rewritten program \cite{DBLP:conf/pods/BancilhonMSU86,DBLP:journals/jlp/BeeriR91,DBLP:journals/jlp/BalbinPRM91,DBLP:conf/pods/StuckeyS94,DBLP:journals/ai/AlvianoFGL12}.
In a nutshell, magic sets introduce rules defining additional atoms, called \emph{magic atoms}, whose intent is to identify relevant atoms to answer the input query, and these magic atoms are added in the bodies of the original rules to restrict the range of the object variables.
Without going into much details, consider a typical recursive definition such as the ancestor relation:
\begin{asp}
  ancestor(X,Y) :- parent(X,Y).
  ancestor(X,Y) :- parent(X,Z), ancestor(Z,Y).
\end{asp}
and a query \lstinline|ancestor(mario,Y)| asking for the ancestors of \lstinline|mario|.
The extension of the \lstinline|ancestor| relation is likely to contain several tuples that are not linked to \lstinline|mario|, and are therefore irrelevant to answer the given query.
To eliminate such a source of inefficiency, magic sets start with \lstinline|m#ancestor#bf(mario)|, the \emph{query seed}, which encodes the relevance of the instances of \lstinline|ancestor(mario,Y)|;
note that the first argument of \lstinline|ancestor| is bound to constant \lstinline|mario|, while the second argument is associated with a free variable, hence the predicate \lstinline|m#ancestor#bf| (first argument bound, second argument free).
After that, magic sets modify the rules defining the intentional predicate \lstinline|ancestor|, and introduce \emph{magic rules} for every occurrence of intentional predicates in the bodies of the modified rules.
The rewritten program is the following:
\begin{asp}
  m#ancestor#bf(mario).
  ancestor(X,Y)    :- m#ancestor#bf(X), parent(X,Y).
  ancestor(X,Y)    :- m#ancestor#bf(X), parent(X,Z), ancestor(Z,Y).
  m#ancestor#bf(Z) :- m#ancestor#bf(X), parent(X,Z).
\end{asp}
and limits the extension of \lstinline|ancestor/2| to the tuples that are relevant to answer the given query.

Magic sets are sound and complete for the language considered in this paper (actually, for a broader language; \citeNP{DBLP:conf/lpnmr/AlvianoGL11}).
However, while on the one hand they are designed to inhibit the source of inefficiency associated with irrelevant atoms, on the other hand they may introduce different sources of inefficiencies, and also produce programs not satisfying the stratification of negation and aggregates.
This paper identifies three of such sources of inefficiency, and propose strategies for their inhibition.
Specifically, the major source of inefficiency is represented by the possible introduction of recursive definitions in the rewritten program.

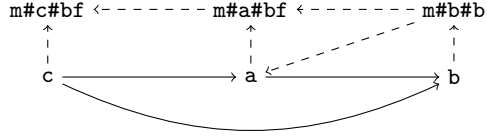
\begin{figure}
    \figrule

    \tikzstyle{node} = [text centered]
    \tikzstyle{line} = [draw, solid]
    \tikzstyle{arc} = [draw, solid, ->]
    \tikzstyle{darc} = [draw, dashed, ->]

    \begin{tikzpicture}[scale=.9]
        \node at (0,0) [node](c) {\lstinline|c|};
        \node at (3,0) [node](a) {\lstinline|a|};
        \node at (6,0) [node](b) {\lstinline|b|};

        \node at (0,1) [node](mc) {\lstinline|m#c#bf|};
        \node at (3,1) [node](ma) {\lstinline|m#a#bf|};
        \node at (6,1) [node](mb) {\lstinline|m#b#b|};

        \draw (c) edge[arc] (a);
        \draw (c) edge[arc, out=-25, in=25-180] (b);
        \draw (a) edge[arc] (b);
        
        \draw (c) edge[darc] (mc);
        \draw (a) edge[darc] (ma);
        \draw (b) edge[darc] (mb);
        
        \draw (ma) edge[darc] (mc);
        \draw (mb) edge[darc] (ma);
        \draw (mb) edge[darc] (a);
        
        % \draw (c) edge[arc, out=45, in=-45] (bc);
        % \draw (a) edge[arc, out=135, in=-135] (ab);
    \end{tikzpicture}
    
    \caption{
        Dependency graphs (defined in Section~\ref{sec:background}) of programs $\Pi_1,\Pi_2,\Pi_3$ (solid arcs) and programs $\Pi_1',\Pi_2',\Pi_3'$ (solid and dashed arcs) from Example~\ref{ex:cycles}.
        All arcs have weight 0, possibly with the exception of the arc connecting \lstinline|a| and \lstinline|b|, which has weight 1 for programs $\Pi_2,\Pi_2',\Pi_3,$ and $\Pi_3'$.
    }\label{fig:pos-cycle}
    \figrule
\end{figure}

\begin{example}[Magic sets may introduce recursive definitions]\label{ex:cycles}
Consider a query \lstinline|c(0,Y)| for the following program $\Pi_1$:
\begin{asp}
  $r_{1\phantom{0}}:\quad$ a(X,Y) :- edb(X,Y), b(X).
  $r_{2\phantom{0}}:\quad$ b(X)   :- edb(X,Y).
  $r_{3\phantom{0}}:\quad$ c(X,Y) :- a(X,Y), b(Y).
\end{asp}
and a possible outcome $\Pi_1'$ of the magic sets rewriting:
\begin{asp}
  $r_{4\phantom{0}}:\quad$ m#c#bf(0).
  $r_{5\phantom{0}}:\quad$ m#a#bf(X) :- m#c#bf(X).
  $r_{6\phantom{0}}:\quad$ m#b#b (Y) :- m#c#bf(X), a(X,Y).
  $r_{7\phantom{0}}:\quad$ m#b#b (X) :- m#a#bf(X), edb(X,Y).
  $r_{8\phantom{0}}:\quad$ a(X,Y) :- m#a#bf(X), edb(X,Y), b(X).
  $r_{9\phantom{0}}:\quad$ b(X)   :- m#b#b (X), edb(X,Y).
  $r_{10}:\quad$ c(X,Y) :- m#c#bf(X), a(X,Y), b(Y).
\end{asp}
In particular, rule $r_6$ is produced while processing rule $r_3$ with variable \lstinline|X| bound from the head atom, and considering variable \lstinline|Y| bound by atom \lstinline|a(X,Y)|.
This is a common strategy, as there is no reason to consider an atom  \lstinline|b(y)| if no instance of \lstinline|a(X,y)| is first computed.
However, as shown in Figure~\ref{fig:pos-cycle}, while all definitions in $\Pi_1$ are non-recursive, $\Pi_1'$ has recursive definitions for \lstinline|a/2| and \lstinline|b/1|, which may deteriorate the performance of the subsequent bottom-up evaluation.
Following the same strategy, for a program $\Pi_2$ comprising $r_2,r_3$ and
\begin{asp}
  $r_{11}:\quad$ a(X,Y) :- edb(X,Y), not b(X).
\end{asp}
the outcome of the magic sets rewriting $\Pi_2'$ comprises rules in $\Pi_1' \setminus \{r_8\}$ and the following rule:
\begin{asp}
  $r_{12}:\quad$ a(X,Y) :- m#a#bf(X), edb(X,Y), not b(X).
\end{asp}
Note that $\Pi_2'$ is not stratified with respect to negation.
Similarly, for $\Pi_3$ comprising $r_2,r_3$ and
\begin{asp}
  $r_{13}:\quad$ a(X,Y) :- edb(X,Y), #sum{1 : b(X)} = 0.
\end{asp}
the magic sets rewriting $\Pi_3'$ comprises rules in $\Pi_1' \setminus \{r_8\}$ and the following rule:
\begin{asp}
  $r_{14}:\quad$ a(X,Y) :- m#a#bf(X), edb(X,Y), #sum{1 : b(X)} = 0.
\end{asp}
Hence, $\Pi_3'$ is not stratified with respect to aggregations.
\hfill$\blacksquare$
\end{example}

A second source of inefficiency that magic sets may introduce is represented by multiple versions of the original rules when the range of variables cannot be eventually restricted.
For example, processing query \lstinline|a(0)| and the following rule:
\begin{asp}
  $r_{15}:\quad$ a(X) :- b(X), a(Y), not c(X,Y).
\end{asp}
necessarily leads to the presence of the following rules in the outcome of magic sets:
\begin{asp}
  $r_{16}:\quad$ m#a#b(0).
  $r_{17}:\quad$ m#a#f :- m#a#b(X).
  $r_{18}:\quad$ a(X)  :- m#a#b(X), b(X), a(Y), not c(X,Y).
  $r_{19}:\quad$ a(X)  :- m#a#f,    b(X), a(Y), not c(X,Y).
\end{asp}
because variable \lstinline|Y| is free when \lstinline|a(Y)| is processed.
Hence, in this case all instances of \lstinline|a/1| in the stable model of the input program are relevant to answer the query in input.
Nevertheless, when such a situation occurs, magic sets already produced restricted versions of the original rules, which are likely to decrease the performance of the subsequent bottom-up evaluation of the rewritten program.

The third source of inefficiency identified in this paper is represented by the possible presence of several copies of the same rule in the rewritten program, which is mainly due to different orders of body literals considered during the application of magic sets.
While this fact is peculiar of one of the possible implementations of magic sets, it is also an opportunity to address a broader source of inefficiency that may already affect the input program, that is, the presence of \emph{subsumed rules}.
In a nutshell, a rule $r$ subsumes another rule $r'$ if the ground instances of $r'$ are included or \emph{less general} than the ground instances of $r$.
For example, \lstinline|q(X) :- p(X,Y)| subsumes \lstinline|q(X) :- p(X,a)|,\linebreak whose ground instances are among those of the first rule, and also \lstinline|q(X) :- p(X,Y), t(X)|, whose ground instances are less general than those of the first rule.

Summarizing the contributions of this paper, the source of inefficiency associated with the introduction of recursive definitions is inhibited by actively monitoring the dependency graph of the rewritten program, so to avoid the creation of new cycles during the production of magic rules (Section~\ref{sec:cycles}).
The other two sources of inefficiency are instead addressed by processing the outcome of magic sets before executing the bottom-up evaluation.
Specifically, if a predicate $p$ is associated with different magic predicates, one of them with all arguments free, the rewritten program is simplified by removing all (useless) rules defining $p$ and whose body contains a magic predicate restricting the range of object variables (Section~\ref{sec:free}).
Concerning subsumed rules, they are identified by means of a backtracking algorithm, whose execution is often prevented by a more efficient but incomplete check based on hashed values and bitwise operations (Section~\ref{sec:subsumption}).
All the proposed strategies are implemented in \textsc{dlv} \cite{DBLP:conf/lpnmr/AlvianoCDFLPRVZ17,DBLP:journals/ki/AdrianACCDFFLMP18,DBLP:conf/lpnmr/LeoneAACCCFFGLC19,DBLP:conf/cilc/LeoneAACCCCFFGL19}, whose magic sets algorithm can be now applied also for programs with stratified aggregates, and assessed empirically on domains involving ontological reasoning (Section~\ref{sec:exp}).

\section{Background}\label{sec:background}

\paragraph{Syntax.}
A \emph{term} is either a constant or an (object) variable.
An \emph{atom} has the form $p(\mathbf{t})$, where $p$ is a \emph{predicate} of arity $n \geq 0$, and $\mathbf{t}$ is a list of $n$ terms.
For a list $\mathbf{t}$, let $|\mathbf{t}|$ denote the length of $\mathbf{t}$, and $\mathbf{t}_i$ denote the $i$-th term of $\mathbf{t}$.
A \emph{literal} is an atom possibly preceded by the \emph{(default) negation} symbol $\mathit{not}$;
atoms are \emph{positive literals}, while atoms preceded by $\mathit{not}$ are \emph{negative literals}.
An \emph{aggregate} has the form $\#\textsc{sum}\{\mathbf{t'} : p(\mathbf{t})\} \odot t$, where $\mathbf{t},\mathbf{t'}$ are lists of terms, $t$ is a term, and $\odot$ is a comparator in $\{<,\leq,=,\neq,\geq,>\}$.
A \emph{rule} has the form
%\begin{align}\label{eq:rule}
    $$\alpha \la \ell_1, \ldots, \ell_n, A_1,\ldots,A_m,$$
%\end{align}
where $\alpha$ is an atom, $n \geq 0$, $m \geq 0$, $\ell_1,\ldots,\ell_n$ are literals, and $A_1,\ldots,A_m$ are aggregates.
For such a rule $r$, define the following notation:
$H(r) := \alpha$, the \emph{head} of $r$;
$B(r) := \{\ell_1, \ldots, \ell_n, A_1,\ldots,A_m\}$, the \emph{body} of $r$;
$B^+(r) := \{\ell_i \mid i \in [1..n], \ell_i$ is a positive literal$\}$;
$B^-(r) := \{\ell_i \mid i \in [1..n], \ell_i$ is a negative literal$\}$;
$B^A(r) := \{A_i \mid i \in [1..m]\}$.
Intuitively, $B(r)$ is interpreted as a conjunction, and we will use $\alpha \la S \wedge S'$ to denote a rule $r$ with $H(r) = \alpha$ and $B(r) = S \cup S'$;
abusing of notation, we also permit $S$ and $S'$ to be literals.
If $B(r)$ is empty, the symbol $\la$ is usually omitted, and the rule is called a \emph{fact}.
A \emph{program} $\Pi$ is a set of rules.
A predicate $p$ occurring in $\Pi$ is said \emph{extensional} if all rules of $\Pi$ with $p$ in their heads are facts;
otherwise, $p$ is said \emph{intentional}.
For any \emph{expression} (atom, literal, aggregate, rule, program) $E$, let $\mathit{At}(E)$ denote the set of atoms occurring in $E$.
In the following, all programs are assumed to satisfy \emph{safety of rules} and \emph{stratification of negation and aggregates}, defined next.

\paragraph{Safety of rules.}
A \emph{global variable} of a rule $r$ is a variable $X$ occurring in $H(r)$, $B^+(r)$, $B^-(r)$, or in an aggregate of the form $\#\textsc{sum}\{\mathbf{t'} : p(\mathbf{t})\} \odot X$ in $B^A(r)$.
All other variables occurring in $r$ are \emph{local variables} (to the aggregates where they occur).
An \emph{assignment variable} of a rule $r$ is a variable $X$ such that $B^A(r)$ contains an aggregate of the form $\#\textsc{sum}\{\mathbf{t'} : p(\mathbf{t})\} = X$.
A global variable $X$ of $r$ is \emph{safe} if $X$ is an assignment variable, or if $X$ occurs in $B^+(r)$.
A local variable $X$ in an aggregate $\#\textsc{sum}\{\mathbf{t'} : p(\mathbf{t})\} \odot t$ of $r$ is \emph{safe} if $X$ occurs in $\mathbf{t}$.
A rule is safe if all of its variables are safe.
A program $\Pi$ satisfies safety of rules if all of its rules are safe.
All rules so far are safe;
an unsafe rule is, for example, \lstinline|a(X,Y) :- b(X), not c(X,Y), #sum{Z : d(X,Y)} > 0|, as in fact the global variable \lstinline|Y| and the local variable \lstinline|Z| are unsafe.

\paragraph{Stratification of negation and aggregates.}
The \emph{dependency graph} $\mathcal{G}_\Pi$ of a program $\Pi$ has nodes for each predicate occurring in $\Pi$, and a weighted arc from $p$ to $q$ if there is a rule $r$ of $\Pi$ such that $p$ occurs in $H(r)$, and $q$ occurs in $B(r)$;
the arc has weight 1 if $q$ occurs in $B(r) \setminus B^+(r)$, and 0 otherwise.
$\Pi$ satisfies stratification of negation and aggregates if $\mathcal{G}_\Pi$ has no cycle involving arcs of positive weight.
Figure~\ref{fig:pos-cycle} shows the dependencies graphs of the programs in Example~\ref{ex:cycles}.

\paragraph{Semantics.}
The \emph{universe} $U_\Pi$ of $\Pi$ is the set comprising all integers, and the constants occurring in $\Pi$.
The \emph{base} $B_\Pi$ of $\Pi$ is the set of atoms constructible from predicates of $\Pi$ with constants in $U_\Pi$.
A \emph{substitution} $\sigma$ is a mapping from variables to variables and $U_\Pi$;
for an expression $E$, let $E\sigma$ be the expression obtained from $E$ by replacing each variable $X$ by $\sigma(X)$.
An expression is \emph{ground} if it contains no global variables.
Let $\mathit{ground}(\Pi)$ be $\bigcup_{r \in \Pi}\{r\sigma \mid \sigma$ is a substitution, and $r\sigma$ is ground$\}$.
An \emph{interpretation} $I$ is a subset of $B_\Pi$.
Relation $\models$ is defined as follows:
for a ground atom $\alpha$, $I \models \alpha$ if $\alpha \in I$, and $I \models \mathit{not}\ \alpha$ if $I \not\models \alpha$;
for an aggregate $A := \#\textsc{sum}\{\mathbf{t'} : p(\mathbf{t})\} \odot t$ occurring in $\mathit{ground}(\Pi)$, $I \models A$ if $\sum_{\mathbf{t'}\sigma : p(\mathbf{t})\sigma \in I}{\mathbf{t'}_\mathtt{1}\sigma} \odot t$;
for a ground rule $r$, $I \models B(r)$ if $I \models \ell$ for all $\ell \in B(r)$, and $I \models r$ if $I \models H(r)$ whenever $I \models B(r)$;
finally, $I \models \mathit{ground}(\Pi)$ if $I \models r$ for all $r \in \mathit{ground}(\Pi)$.
The \emph{(FLP) reduct} of $\Pi$ with respect to $I$, denoted $\Pi^I$, is the program obtained from $\Pi$ by removing rules with false bodies, that is, $\Pi^I := \{r \in \Pi \mid I \models B(r)\}$ \cite{DBLP:journals/ai/FaberPL11}.
Given a program $\Pi$, the \emph{stable model} of $\Pi$ is the unique interpretation $I$ such that $I \models \mathit{ground}(\Pi)$, and there is no $J \subset I$ such that $J \models \mathit{ground}(\Pi)^I$;
let $\mathit{SM}(\Pi)$ denote the stable model of $\Pi$.
(The stable model of $\Pi$ can be computed bottom-up as described in the introduction.
A formal definition of such a procedure is out of the scope of this paper.)

\begin{example}
Consider the following program in the context of an online shopping site:
\begin{asp}
  order(o1). item(o1,i1,20). item(o1,i2,20).
  order(o2). cancelled(o2).
  total_cost(S) :- order(O), not cancelled(O), #sum{P,I : item(O,I,P)} = S.
\end{asp}
The stable model of the above program contains facts and \lstinline|total_cost(40)|, as indeed the only ground rule with true, nonempty body is the following:
\begin{asp}
  total_cost(40) :- order(o1), not cancelled(o1), #sum{P,I : item(o1,I,P)} = 40.
\end{asp}
In particular, note that for $\sigma(\mathtt{O}) \notin \{\mathtt{o1},\mathtt{o2}\}$ literal \lstinline|order(O)|$\sigma$ is false, for $\sigma(\mathtt{O}) = \mathtt{o2}$ literal \lstinline|not cancelled(o2)| is false, and for $\sigma(\mathtt{O}) = \mathtt{o1}$ and $\sigma(\mathtt{S}) \neq \mathtt{40}$ the aggregate is false.
\hfill$\blacksquare$
\end{example}

\paragraph{Queries and magic sets.}
A query is an atom $q(\mathbf{t})$.
Let $\mathit{answer}(q(\mathbf{t}),\Pi)$ be $\{\mathbf{t}\sigma \mid q(\mathbf{t})\sigma \in \mathit{SM}(\Pi)\}$, that is, the answer to the query $q(\mathbf{t})$ over the program $\Pi$ is the set of ground instances of $q(\mathbf{t})$ in the stable model of $\Pi$.
The magic sets algorithm aims at transforming program $\Pi$ into a program $\Pi'$ such that $\mathit{answer}(q(\mathbf{t}),\Pi) = \mathit{answer}(q(\mathbf{t}),\Pi')$, and $\mathit{SM}(\Pi') \cap \mathit{At}(\Pi) \subseteq \mathit{SM}(\Pi)$;
in words, the two programs have the same answer to the query $q(\mathbf{t})$, but the stable model of $\Pi'$ only contains atoms that link facts to the query.
The algorithm relies on \emph{adornments} and \emph{magic atoms} to represent binding information that a top-down evaluation of the query would produce.

\begin{definition}[Adornments and magic atoms]\label{def:adornments}
An adornment for a predicate $p$ of arity $k$ is any string $\mathbf{s}$ of length $k$ over the alphabet $\{b,f\}$.
The $i$-th argument of $p$ is \emph{bound} with respect to $\mathbf{s}$ if $\mathbf{s}_i = b$, and \emph{free} otherwise, for all $i \in [1..k]$.
For an atom $p(\mathbf{t})$, let $p^{\mathbf{s}}(\mathbf{t})$ be the (magic) atom $m\#p\#\mathbf{s}(\mathbf{t'})$, where $m\#p\#\mathbf{s}$ is a predicate not occurring in the input program, and $\mathbf{t'}$ contains all terms in $\mathbf{t}$ associated with bound arguments according to $\mathbf{s}$.
\end{definition}

\begin{definition}[Sideways information passing strategy; SIPS]\label{def:sip}
A SIPS for a rule $r$ with respect to an adornment $\mathbf{s}$ for $H(r)$ is a pair $(\prec,\mathit{bnd})$, where $\prec$ is a strict partial order over $\{H(r)\} \cup B(r)$, and $\mathit{bnd}$ maps $\ell \in \{H(r)\} \cup B(r)$ to the variables of $\ell$ that are made bound after processing $\ell$.
Moreover, a SIPS satisfies the following conditions:
\begin{itemize}
\item $H(r) \prec \ell$ for all $\ell \in B(r)$ (binding information originates from head atoms);
\item $\ell \prec \ell'$ and $\ell \neq H(r)$ implies that either $\ell \in B^+(r)$ or $\ell$ is an aggregate with assignment (new bindings are created only by positive literals and assignments);
\item $\mathit{bnd}(H(r))$ contains the variables of $H(r)$ associated with bound arguments according to $\mathbf{s}$;
\item $\mathit{bnd}(\ell) = \emptyset$ if $\ell$ is a negative literal, or an aggregate without assignment variable;
\item $\mathit{bnd}(\ell) \subseteq \{X\}$ if $\ell$ is an aggregate with assignment variable $X$.
\end{itemize}
\end{definition}

\begin{example}[Magic atoms and SIPS]
According to Definition~\ref{def:adornments}, \lstinline|c$^\mathit{bf}$(0,Y)| is the magic atom \lstinline|m#c#bf(0)|.
Using the notation introduced in Definition~\ref{def:sip}, the SIPS for $r_3$ with respect to the adornment $\mathit{bf}$ adopted in Example~\ref{ex:cycles} is such that 
$c(X,Y) \prec a(X,Y) \prec b(Y)$,
$\mathit{bnd}(c(X,Y)) = \{X\}$,
$\{Y\} \subseteq \mathit{bnd}(a(X,Y)) \subseteq \{X,Y\}$ (i.e., variable \lstinline|Y| is bound after processing \lstinline|a(X,Y)|), and
$\emptyset \subseteq \mathit{bnd}(b(Y)) \subseteq \{Y\}$ (i.e., whether \lstinline|Y| is bound after processing \lstinline|b(Y)| is irrelevant).
\hfill$\blacksquare$
\end{example}

\begin{algorithm}[t]
    \caption{MS($Q(\mathbf{T})$: a query atom, $\Pi$: a program)}\label{alg:MS}
    Let $\mathbf{s}$ be such that $|\mathbf{s}| = |\mathbf{T}|$, and $\mathbf{s}_i = b$ if $\mathbf{T}_i$ is a constant, and \!$f$\! otherwise, for all $i \in [1..|\mathbf{s}|]$\;
    $\Pi' := \{Q^\mathbf{s}(\mathbf{T}).\}$\tcp*{rewritten program: start with the magic seed}
    $S := \{\tuple{Q,\mathbf{s}}\}$\tcp*{set of produced adorned predicates}
    $D := \emptyset$\tcp*{set of processed (or done) adorned predicates}
    \While{$S \neq D$}{
        $\tuple{q,\mathbf{s}} := $ any element in $S \setminus D$\tcp*{select an undone adorned predicate}
        \ForEach{$r \in \Pi$ such that $H(r) = q(\mathbf{t})$ for some list $\mathbf{t}$ of terms}{
        $\Pi' := \Pi' \cup \{q(\mathbf{t}) \la q^\mathbf{s}(\mathbf{t}) \wedge B(r).\}$\tcp*{restrict range of variables}
            Let $(\prec,\mathit{bnd})$ be the SIPS for $r$ with respect to $\mathbf{s}$\;
            \ForEach{$\ell \in B(r)$ such that $p(\mathbf{t'}) \in \mathit{At}(\ell)$ and $p$ is an intentional predicate of $\Pi$}{
                Let $\mathbf{s'}$ be such that $|\mathbf{s'}| = |\mathbf{t'}|$, and $\mathbf{s'_\mathit{i}} = b$ if $\mathbf{t'_\mathit{i}}$ is a constant or belongs to $\mathit{bnd}(\ell')$ for some $\ell' \prec \ell$, and $f$ otherwise, for all $i \in [1..|\mathbf{s'}|]$\;
                $\Pi' := \Pi' \cup \{p^\mathbf{s'}(\mathbf{t'}) \la q^s(\mathbf{t}) \wedge \{\ell' \in B(r) \mid \ell' \prec \ell\}.\}$\tcp*{add magic rule}
                $S := S \cup \{\tuple{p,\mathbf{s'}}\}$\tcp*{keep track of produced adorned predicates}
            }
        }
        $D := D \cup \{\tuple{q,\mathbf{s}}\}$\tcp*{flag the adorned predicate as done}
    }
    \Return{$\Pi'$}\;
\end{algorithm}
The magic sets procedure is reported as Algorithm~\ref{alg:MS}.
It starts by producing the \emph{magic seed}, obtained from the predicate and the constants in the query.
After that, the algorithm processes each produced adorned predicate:
each rule defining the predicate is modified so to restrict the range of the head variables to the tuples that are relevant to answer the query;
such a relevance is encoded by the magic rules, which are produced for all intentional predicates in the bodies of the modified rules.

\begin{restatable}[Theorem~5 of Alviano et al. 2011]{proposition}{PropCorrectness}\label{prop:correctness}
%\citeNP{DBLP:conf/lpnmr/AlvianoGL11}
Let $q(\mathbf{t})$ be a query for a program $\Pi$, and $\Pi'$ be the output of $\textnormal{MS}(q(\mathbf{t}),\Pi)$.
Thus, $\mathit{answer}(q(\mathbf{t}),\Pi)$ and $\mathit{answer}(q(\mathbf{t}),\Pi')$ are equal.
\end{restatable}

\section{Improved strategies for the magic sets algorithm}\label{sec:strategies}

The three sources of inefficiency of magic sets that have been identified in the introduction are detailed and addressed in this section.

\subsection{Inhibit new cycles}\label{sec:cycles}

Magic sets may introduce new cycles in the dependency graph of the processed program, as shown in Example~\ref{ex:cycles}.
Such new cycles are due to the binding information passed by body literals to other body literals, and therefore strictly dependent from the adopted SIPS.
In fact, new cycles can be inhibited by a drastic restriction on all SIPS $\tuple{\prec,\mathit{bnd}}$ enforcing $\ell \nprec \ell$' for all $\ell,\ell'$ in $B(r)$:
this way, all magic rules would contain only magic atoms, and therefore no arc from magic predicates to original predicates would be introduced in the dependency graph.
However, the drastic restriction is likely to significantly reduce the benefit of magic sets, as the stronger the restriction on SIPS is, the more atoms are considered relevant to answer a given query.
Hence, the goal of this section is to introduce a more relaxed restriction on SIPS, which just prevents the creation of new cycles, but still admit the introduction of new dependencies.

\begin{algorithm}[t]
    \caption{MS-RS($Q(\mathbf{T})$: a query atom, $\Pi$: a program)}\label{alg:ms-rs}
    Let $\mathbf{s}$ be such that $|\mathbf{s}| = |\mathbf{T}|$, and $\mathbf{s}_i = b$ if $\mathbf{T}_i$ is a constant, and \!$f$\! otherwise, for all $i \in [1..|\mathbf{s}|]$\;
    $\Pi' := \{Q^\mathbf{s}(\mathbf{T}).\}$\tcp*{rewritten program: start with the magic seed}
    $S := \{\tuple{Q,\mathbf{s}}\}$\tcp*{set of produced adorned predicates}
    $D := \emptyset$\tcp*{set of processed (or done) adorned predicates}
    $G := \mathcal{G}_\Pi \cup \{\tuple{p,m\#p} \mid p$ is a predicate occurring in $\Pi\}$\tcp*{monitor SCCs}
    \While{$S \neq D$}{
        $\tuple{q,\mathbf{s}} := $ any element in $S \setminus D$\tcp*{select an undone adorned predicate}
        \ForEach{$r \in \Pi$ such that $H(r) = q(\mathbf{t})$ for some list $\mathbf{t}$ of terms}{
        $\Pi' := \Pi' \cup \{q(\mathbf{t}) \la q^\mathbf{s}(\mathbf{t}) \wedge B(r).\}$\tcp*{restrict range of variables}
            Let $(\prec,\mathit{bnd})$ be the SIPS for $r$ with respect to $\mathbf{s}$\;
            \ForEach{$\ell \in B(r)$ such that $p(\mathbf{t'}) \in \mathit{At}(\ell)$ and $p$ is an intentional predicate of $\Pi$}{
                $G := G \cup \{\tuple{m\#p, m\#q}\}$\;
                $B := \emptyset$\tcp*{restrict SIPS to preserve strongly connected comp.}
                \ForEach{$\ell' \in B(r)$ such that $\ell' \prec \ell$ and $p'(\mathbf{t''}) \in \mathit{At}(\ell')$}{
                    \If{$\{C \cap \mathit{At}(\Pi) \mid C \in \mathit{SCCs}(G \cup \{\tuple{m\#p, p'}\})\} = \mathit{SCCs}(\mathcal{G}_\Pi)$\label{alg:ms-rs:ln:restrict}}{
                        $B := B \cup \{\ell'\}$;\quad $G := G \cup \{\tuple{m\#p, p'}\}$\;
                    }
                }
                Let $\mathbf{s'}$ be such that $|\mathbf{s'}| = |\mathbf{t'}|$, and $\mathbf{s'_\mathit{i}} = b$ if $\mathbf{t'_\mathit{i}}$ is a constant or belongs to $\mathit{bnd}(\ell')$ for some $\ell' \in \{H(r)\} \cup B$ such that $\ell' \prec \ell$, and $f$ otherwise, for all $i \in [1..|\mathbf{s'}|]$\;
                $\Pi' := \Pi' \cup \{p^\mathbf{s'}(\mathbf{t'}) \la q^s(\mathbf{t}) \wedge B.\}$\tcp*{add magic rule}
                $S := S \cup \{\tuple{p,\mathbf{s'}}\}$\tcp*{keep track of produced adorned predicates}
            }
        }
        $D := D \cup \{\tuple{q,\mathbf{s}}\}$\tcp*{flag the adorned predicate as done}
    }
    \Return{$\Pi'$}\;
\end{algorithm}

For a graph $G$ and a set of arcs $E$, let $G \cup E$ denote the graph obtained from $G$ by adding each arc in $E$.
Moreover, let $\mathit{SCCs}(G)$ be the set of \emph{strongly connected components (SCC)} of $G$, where a SCC of $G$ is a maximal set $C$ of nodes of $G$ such that $G$ contains a path from every $p \in C$ to every $q \in C \setminus \{p\}$.
A revised version of magic sets enforcing a restriction on SIPS is shown as Algorithm~\ref{alg:ms-rs}.
Note that lines 5 and 12--16 implement a restriction of SIPS guaranteeing that no strongly connected components of $\mathcal{G}_\Pi$ are merged during the application of magic sets.
Specifically, a graph $G$ is initialized with the arcs of $\mathcal{G}_\Pi$ and arcs connecting each predicate $p$ with a \emph{representative magic predicate} $m\#p$ (line~5).
After that, before creating a new magic rule, elements of $B(r)$ that would cause a change in the strongly connected components of $G$ are discarded (lines~13--16).
Graph $G$ is updated with new arcs involving original predicates and representative magic predicates, so that it represents a superset of the graph obtained from $\mathcal{G}_{\Pi'}$ by merging all pairs of nodes of the form $m\#p\#\mathbf{s}$, $m\#p\#\mathbf{s'}$.

\pagebreak
\begin{example}
Consider $\Pi_1$, query \lstinline|c(0,Y)|, and SIPS from Example~\ref{ex:cycles}.
Algorithm~\ref{alg:ms-rs} returns the following program:
\begin{asp}
 $r_{20}:\ $ m#c#bf(0).
 $r_{21}:\ $ m#a#bf(X) :- m#c#bf(X).
 $r_{22}:\ $ m#b#f     :- m#c#bf(X).
 $r_{23}:\ $ m#b#b (X) :- m#a#bf(X), edb(X,Y).
 $r_{24}:\ $ a(X,Y) :- m#a#bf(X), edb(X,Y),b(X).
 $r_{25}:\ $ b(X)   :- m#b#f,     edb(X,Y).
 $r_{26}:\ $ b(X)   :- m#b#b(X),  edb(X,Y).
 $r_{27}:\ $ c(X,Y) :- m#c#bf(X), a(X,Y), b(Y).
\end{asp}
Note that rule $r_6$ from Example~\ref{ex:cycles} is replaced by rule $r_{22}$, so to avoid the creation of a cycle involving \lstinline|a| and \lstinline|b|.
Note also that predicate \lstinline|b| is now associated with two magic predicates, which may reduce the performance of a bottom-up evaluation;
this source of inefficiency is addressed in the next section.
\hfill$\blacksquare$
\end{example}

\begin{restatable}{theorem}{ThmCycles}\label{thm:cycles}
Let $q(\mathbf{t})$ be a query for a program $\Pi$, and $\Pi'$ be the output of $\text{MS}(q(\mathbf{t}),\Pi)$ with restricted SIPS.
Thus, $\mathit{answer}(q(\mathbf{t}),\Pi)$ and $\mathit{answer}(q(\mathbf{t}),\Pi')$ are equal.
Moreover, if $C' \in \mathit{SCCs}(\Pi')$, then there is $C \in \mathit{SCCs}(\Pi)$ such that $C' \cap \mathit{At}(\Pi) \subseteq C$.
\end{restatable}
\begin{proof}
Equality of $\mathit{answer}(q(\mathbf{t}),\Pi)$ and $\mathit{answer}(q(\mathbf{t}),\Pi)$ is a consequence of the correctness of magic sets for any choice of SIPS (Proposition~\ref{prop:correctness}).
In fact, the restriction on SIPS applied by algorithm $\text{MS-RS}$ still results into SIPS.
For $C' \in \mathit{SCCs}(\mathcal{G}_{\Pi'})$, we shall show that there is $C \in \mathit{SCCs}(\mathcal{G}_\Pi)$ such that $C' \cap \mathit{At}(\Pi) \subseteq C$.
Actually, there is $C \in \mathit{SCCs}(G)$ such that $C' \cap \mathit{At}(\Pi) \subseteq C \cap \mathit{At}(\Pi)$.
Hence, the claim follows from the fact that $C \cap \mathit{At}(\Pi) \in \mathit{SCCs}(\mathcal{G}_\Pi)$ is enforced by the condition in line~\ref{alg:ms-rs:ln:restrict} of Algorithm~\ref{alg:ms-rs}.
\end{proof}

An immediate consequence of the above theorem is that magic sets with restricted SIPS are a closed rewriting for the class of programs with stratified negation and aggregations.

\subsection{Handle full-free adornments}\label{sec:free}

\begin{algorithm}[t]
    \caption{FullFree($\Pi$: a program obtained by executing magic sets)}\label{alg:full-free}
    \ForEach{$m\#p\#f \cdots f$ occurring in $\Pi$}{
        \ForEach{$m\#p\#\mathbf{s}$ occurring in $\Pi$ such that $\mathbf{s} \neq f \cdots f$}{
            remove all rules of $\Pi$ having $m\#p\#\mathbf{s}$ in their bodies\;
            replace $m\#p\#\mathbf{s}(\mathbf{t})$ by $m\#p\#f \cdots f$ in all rule heads of $\Pi$\;
        }
    }
    \Return{$\Pi$}\;
\end{algorithm}

Adornments containing only $f$s are produced in presence of predicates whose arguments are all free.
In such cases, all of the extension of the predicate in the stable model of the input program is relevant to answer the given query.
It turns out that the range of the object variables of all rules defining such predicates cannot be restricted, and indeed the magic sets rewriting includes a copy of these rules with a magic atom obtained from the full-free adornment.
Possibly, the magic sets rewriting includes other copies of these rules obtained by different adornments, which can be removed if magic rules are properly modified.
Specifically, magic rules associated with predicates for which a full-free adornment has been produced have to become definitions of the magic atom obtained from the full-free adornment.
The strategy is summarized in Algorithm~\ref{alg:full-free}, and can be efficiently implemented in two steps:
a first linear traversal of the program to identify predicates of the form $m\#p\#f \cdots f$ and to flag predicate $p$;
a second linear traversal of the program to remove and rewrite rules with predicate $m\#p\#\mathbf{s}$, for all flagged predicates $p$.

\begin{example}
Consider rule $r_{15}$ from the introduction, \lstinline|a(X) :- b(X), a(Y), not c(X,Y)|,
%\begin{asp}
%  $r_{15}:\quad$ a(X) :- b(X), a(Y), not c(X,Y).
%\end{asp}
and its magic sets rewriting with respect to query \lstinline|a(0)|:
\begin{asp}
 $r_{16}:\ $ m#a#b(0).
 $r_{17}:\ $ m#a#f :- m#a#b(X).
 $r_{18}:\ $ a(X)  :- m#a#b(X), b(X), a(Y), not c(X,Y).
 $r_{19}:\ $ a(X)  :- m#a#f,    b(X), a(Y), not c(X,Y).
\end{asp}
Algorithm~\ref{alg:full-free} removes rules $r_{17}$ and $r_{18}$ because of \lstinline|m#a#b(X)| in their bodies, and replaces rule $r_{16}$ with the fact \lstinline|m#a#f|.
\hfill$\blacksquare$
\end{example}

\begin{restatable}{theorem}{ThmFullFree}
Let $q(\mathbf{t})$ be a query for a program $\Pi$, and $\Pi'$ be the output of $\textnormal{FullFree}(\textnormal{MS}(q(\mathbf{t}),\Pi))$.
Thus, $\mathit{answer}(q(\mathbf{t}),\Pi)$ and $\mathit{answer}(q(\mathbf{t}),\Pi')$ are equal.
\end{restatable}
\begin{proof}
Let $I$ be $\mathit{SM}(\Pi'')$.
The stable model of $\Pi'$ is obtained from $I$ by performing the following operation for all $m\#p\#f \cdots f$ occurring in $\Pi''$:
replace all instances of $m\#p\#\mathbf{s}$ by $m\#p\#f \cdots f$.
\end{proof}

\subsection{Efficiently detect subsumed rules}\label{sec:subsumption}

A rule $r$ \emph{subsumes} a rule $r'$, denoted $r \sqsubseteq r'$, if there is a substitution $\sigma$ such that $H(r)\sigma = H(r')$ and $B(r)\sigma \subseteq B(r')$.
Subsumed rules are redundant in the sense that any atom derivable from $r'$ is also derived from $r$ if $r \sqsubseteq r'$;
indeed, for any substitution $\theta$ and interpretation $I$ such that $B(r')\theta$ is ground and $I \models B(r')\theta$, it holds that $B(r)\sigma\theta$ is ground, $I \models B(r)\sigma\theta$ (because $B(r)\sigma\theta \subseteq B(r')\theta$), and $H(r)\sigma\theta = H(r')\theta$.
Hence, $r \sqsubseteq r'$ implies $\mathit{SM}(\Pi) = \mathit{SM}(\Pi \setminus \{r'\})$, and therefore all subsumed rules can be removed from a program before starting its bottom-up evaluation.
However, checking subsumption is NP-complete in general, and therefore computationally expensive if ran for all pairs of rules in a program.

\begin{algorithm}[t]
    \caption{Subsumption($\Pi$: a program)}\label{alg:subsumption}
    \ForEach{distinct $r,r' \in \Pi$ such that $\mathit{hash}(r)\ \&\ \mathit{hash}(r') = \mathit{hash}(r)$}{
        \lIf{subsumes($r,r'$)}{remove $r'$ from $\Pi$}
    }
    \Return{$\Pi$}\;
\end{algorithm}

\begin{function}[t]
    \caption{Subsumes($r$, $r'$)}\label{fun:subsumption}
    $S := [\tuple{\text{OneWayUnify}(H(r),\ H(r')),\ B(r)}]$\;
    \While{$S \neq \emptyset$}{
        $\tuple{\sigma,\ B} := S.\text{pop}()$\;
        \If{$\sigma$ is a function}{
            \lIf{$B = \emptyset$}{\Return{true}}
            \lForEach{$\ell \in B$ and $\ell' \in B(r')$}{
                $S.\text{push}(\tuple{\sigma \cup \text{OneWayUnify}(\ell,\ \ell'),\ B \setminus \{\ell\}})$%
            }
        }
    }
    \Return{false}\;
\end{function}

\begin{function}[t]
    \caption{OneWayUnify($\ell$, $\ell'$)}\label{fun:unify}
    \lIf{$\ell$ and $\ell'$ have different predicates, or are not both positive literals, negative literals, or aggregates}{\Return{$\{X \mapsto 0, X \mapsto 1\}$}}
    Let $\mathbf{t}$ and $\mathbf{t'}$ be the terms in $\ell$ and $\ell'$ (for aggregates, symbol : is considered as a constant)\;
    \lIf{$|\mathbf{t}| \neq |\mathbf{t'}|$, or $\exists i \in [1..|\mathbf{t}|]$ s.t. $\mathbf{t_\mathit{i}}$ is a constant and $\mathbf{t_\mathit{i}} \neq \mathbf{t'_\mathit{i}}$}{\Return{$\{X \mapsto 0, X \mapsto 1\}$}}
    \Return{$\{\mathbf{t_\mathit{i}} \mapsto \mathbf{t'_\mathit{i}} \mid i \in [1..|\mathbf{t}|], \mathbf{t_\mathit{i}}$ is a variable$\}$}\tcp*{possibly a function}
\end{function}

The number of performed checks is significantly reduced by means of an hash function that associates each rule with a bit string of fixed length and satisfying the following invariant:
\begin{align}\label{eq:hash}
\text{if } \mathit{hash}(r)\ \&\ \mathit{hash}(r') \neq \mathit{hash}(r), \text{ then } r \not\sqsubseteq r'
\end{align}
where $\&$ is the bitwise \textsc{and} operator.
Specifically, the hash value associated with a rule is designed to be a string of 64 bits computed as follows from the less significant bits of predicate ids and constant ids occurring in the rule:
8 bits for the bitwise \textsc{or} of predicate ids in $H(r)$ (only one predicate for the language considered in this paper);
8 bits for the bitwise \textsc{or} of constants ids in $H(r)$;
16 bits for the bitwise \textsc{or} of predicate ids in $B^+(r) \cup B^A(r)$;
16 bits for the bitwise \textsc{or} of constant ids in $B^+(r) \cup B^A(r)$;
8 bits for the bitwise \textsc{or} of predicate ids in $B^-(r)$;
8 bits for the bitwise \textsc{or} of constants ids in $B^-(r)$.

The idea underlying the above hash function is that all constants and predicates occurring in $H(r)$, $B^+(r) \cup B^A(r)$ and $B^-(r)$ have to also occur in $H(r')$, $B^+(r') \cup B^A(r')$ and $B^-(r')$ in order to have $r \sqsubseteq r'$.
The invariant (\ref{eq:hash}) eventually detects pairs of rules not satisfying this property, so to avoid the more expensive backtracking procedure for them.
Algorithm~\ref{alg:subsumption} summarizes the strategy implemented for removing subsumed rules from programs:
when the condition on the hash values is satisfied, use backtracking to build a substitution $\sigma$.

\begin{example}
Consider the following rules from the introduction:
\begin{asp}
  $r:$ q(X) :- p(X,Y). $\qquad r':$ q(X) :- p(X,a). $\qquad r'':$ q(X) :- p(X,Y), t(X).
\end{asp}
and the following predicate and constant ids:
$\mathit{id}(q) = 01, \mathit{id}(p) = 10, \mathit{id}(t) = 11, \mathit{id}(a) = 01$, where for simplicity only 2 bits are used.
The hash values of the rules above (using only 2 bits for each portion of the hash value) are the following:
\begin{itemize}
\item $\mathit{hash}(r\phantom{''}) = 01 00 10 00 00 00$;
\item $\mathit{hash}(r'\phantom{'}) = 01 00 10 01 00 00$;
\item $\mathit{hash}(r'') = 01 00 11 00 00 00$.
\end{itemize}
Note that $\mathit{hash}(r')\ \&\ \mathit{hash}(r'') = 01 00 10 00 00 00 \neq \mathit{hash}(r')$, and in fact $r' \not\sqsubseteq r''$.
%Similarly, $\mathit{hash}(r')\ \&\ \mathit{hash}(r) = 01 00 10 00 00 00 \neq \mathit{hash}(r')$, and $r' \not\sqsubseteq r$.
On the other hand, $\mathit{hash}(r)\ \&\ \mathit{hash}(r') = 01 00 10 00 00 00 = \mathit{hash}(r)$, and $r \sqsubseteq r'$.
%Similarly, $\mathit{hash}(r)\ \&\ \mathit{hash}(r'') = 01 00 10 00 00 00 = \mathit{hash}(r)$, and $r \sqsubseteq r''$.
\hfill$\blacksquare$
\end{example}

\begin{restatable}{theorem}{ThmSubsumption}
Invariant (\ref{eq:hash}) is satisfied by the proposed hash function.
\end{restatable}
\begin{proof}
Let $\mathit{hash}(r)\ \&\ \mathit{hash}(r') \neq \mathit{hash}(r)$.
Hence, there is $i \in [1..64]$ such that $\mathit{hash}(r) = 1$ and $\mathit{hash}(r') = 0$.
If $i \in [1..8]$, then predicate ids of $H(r)$ and $H(r')$ disagree on their less significant 8 bits, and therefore they are necessarily different predicates;
thus, $r \not\sqsubseteq r'$ holds.
If $i \in [9..16]$, then $H(r)$ contains a constant whose $(i-8)$-th less significant bit is 1, while no constant in $H(r')$ has this property;
it turns out that $H(r)$ has a constant not occurring in $H(r')$, and therefore $r \not\sqsubseteq r'$ holds also in this case.
The remaining cases are similar.
\end{proof}

\section{Experiment}\label{sec:exp}

The proposed enhancements are implemented in \textsc{i-dlv}~1.1.4, and compared against the performance of the previous magic sets rewriting implemented in \textsc{i-dlv}~1.1.3.
Binaries are available at \url{https://github.com/DeMaCS-UNICAL/I-DLV/releases}.
The experiment comprises synthetic instances from Example~\ref{ex:cycles} with facts \lstinline|edb(0..1000000*size)|, where \lstinline|size| ranges in $[1..10]$, to show the potential impact of the prevention of new cycles.
Additional instances are obtained from LUBM (\url{http://swat.cse.lehigh.edu/projects/lubm/}) by generating instances with \lstinline|50*size| universities, where \lstinline|size| ranges in $[1..20]$, with the aim to measure the impact of the hashing technique to prevent subsumption checks.
All tests were run on a Dell server with 8 CPU Intel Xeon Gold 6140 2.30GHz, RAM 297GB, and HDD 3.29TB 7200rpm.
Each test was limited to 1200 seconds of execution time and 250GB of memory consumption.
%Testcases and binaries are available online at the following URL:\linebreak
%\url{https://alviano.com/iclp2019}.
%FIXME: remove this link and resort the link to the binaries in the final version

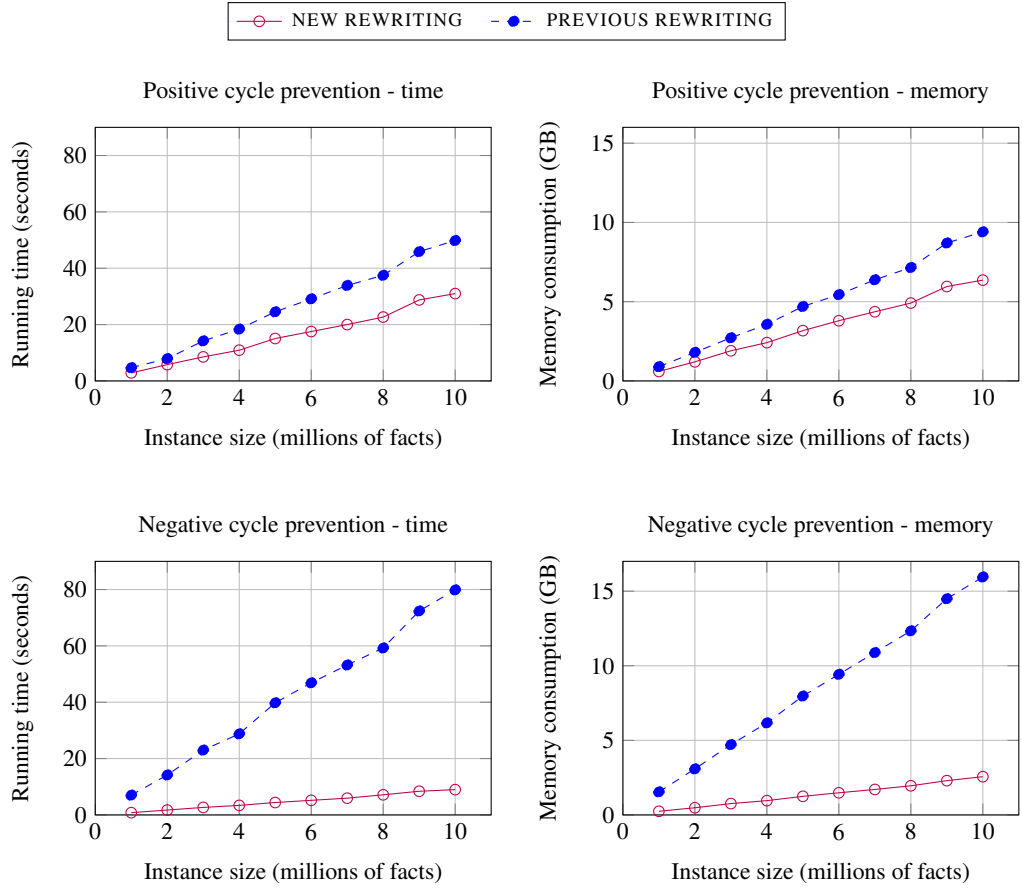
\begin{figure}[t]
    \figrule
    \begin{center}
        \begin{tikzpicture}[scale=1]
            \begin{axis}[
                hide axis,
                scale only axis,
                width=.5\textwidth,
                height=0.1\textwidth,
                xmin=10,
                xmax=50,
                ymin=0,
                ymax=0.4,
                legend style={align=left, legend columns=-1, legend style={column sep=1em}}
            ]
                \addlegendimage{mark size=2pt, color=purple, mark=o}
                \addlegendentry{\textsc{\!\!\!\!\!new rewriting}}
    
                \addlegendimage{mark size=2pt, color=blue, mark=*, dashed}
                \addlegendentry{\textsc{\!\!\!\!\!previous rewriting}}
            \end{axis}
        \end{tikzpicture}
    \end{center}
    \vspace{-2em}
    \begin{tikzpicture}[scale=1]
        \begin{axis}[
            scale only axis,
            width=0.39\textwidth,
            height=0.25\textwidth,
            y label style = {at={(axis description cs:0.05,0.5)}},
            xlabel={Instance size (millions of facts)},
            ylabel={Running time (seconds)},
            xmin=0,
            xmax=11,
            ymin=0,
            ymax=90,
            grid=both,
            title={Positive cycle prevention - time},
        ]
            \addplot [mark size=2pt, color=purple, mark=o] [unbounded coords=jump] table[col sep=tab, skip first n=1, x index=0, y index=1] {scalability-positive.csv};
            %\addlegendentry{\textsc{new rewriting}}

            \addplot [mark size=2pt, color=blue, mark=*, dashed] [unbounded coords=jump] table[col sep=tab, skip first n=1, x index=0, y index=2] {scalability-positive.csv};
            %\addlegendentry{\textsc{previous rewriting}}
        \end{axis}
    \end{tikzpicture}
    \hfill
    \begin{tikzpicture}[scale=1]        
        \begin{axis}[
            scale only axis,
            width=0.39\textwidth,
            height=0.25\textwidth,
            y label style = {at={(axis description cs:0.05,0.5)}},
            xlabel={Instance size (millions of facts)},
            ylabel={Memory consumption (GB)},
            xmin=0,
            xmax=11,
            ymin=0,
            ymax=16,
            grid=both,
            title={Positive cycle prevention - memory},
        ]
            \addplot [mark size=2pt, color=purple, mark=o] [unbounded coords=jump] table[col sep=tab, skip first n=1, x index=0, y index=3] {scalability-positive.csv};
            %\addlegendentry{\textsc{new rewriting}}

            \addplot [mark size=2pt, color=blue, mark=*, dashed] [unbounded coords=jump] table[col sep=tab, skip first n=1, x index=0, y index=4] {scalability-positive.csv};
            %\addlegendentry{\textsc{previous rewriting}}
        \end{axis}
    \end{tikzpicture}    

    \vspace{2em}

    \begin{tikzpicture}[scale=1]
        \begin{axis}[
            scale only axis,
            width=0.39\textwidth,
            height=0.25\textwidth,
            y label style = {at={(axis description cs:0.05,0.5)}},
            xlabel={Instance size (millions of facts)},
            ylabel={Running time (seconds)},
            xmin=0,
            xmax=11,
            ymin=0,
            ymax=90,
            grid=both,
            title={Negative cycle prevention - time},
        ]
            \addplot [mark size=2pt, color=purple, mark=o] [unbounded coords=jump] table[col sep=tab, skip first n=1, x index=0, y index=1] {scalability-negative.csv};
            %\addlegendentry{\textsc{new rewriting}}

            \addplot [mark size=2pt, color=blue, mark=*, dashed] [unbounded coords=jump] table[col sep=tab, skip first n=1, x index=0, y index=2] {scalability-negative.csv};
            %\addlegendentry{\textsc{previous rewriting}}
        \end{axis}
    \end{tikzpicture}
    \hfill
    \begin{tikzpicture}[scale=1]        
        \begin{axis}[
            scale only axis,
            width=0.39\textwidth,
            height=0.25\textwidth,
            y label style = {at={(axis description cs:0.05,0.5)}},
            xlabel={Instance size (millions of facts)},
            ylabel={Memory consumption (GB)},
            xmin=0,
            xmax=11,
            ymin=0,
            ymax=17,
            grid=both,
            title={Negative cycle prevention - memory},
        ]
            \addplot [mark size=2pt, color=purple, mark=o] [unbounded coords=jump] table[col sep=tab, skip first n=1, x index=0, y index=3] {scalability-negative.csv};
            %\addlegendentry{\textsc{new rewriting}}

            \addplot [mark size=2pt, color=blue, mark=*, dashed] [unbounded coords=jump] table[col sep=tab, skip first n=1, x index=0, y index=4] {scalability-negative.csv};
            %\addlegendentry{\textsc{previous rewriting}}
        \end{axis}
    \end{tikzpicture}

    \caption{Scalability results with respect to time (left) and memory (right)}\label{fig:scalability}
    \figrule
\end{figure}

Concerning the scalability tests, time and memory usage are plotted on Figure~\ref{fig:scalability}.
For program $\Pi_1$, the execution time of the new rewriting is around 62\% of the execution time of the previous rewriting on average;
similarly, the new rewriting only used around 68\% of the memory required by the previous rewriting.
These results confirm that the proposed restriction to SIPS may lead computational advantages also for positive programs.
The advantage is much more evident for program $\Pi_2$, that is, the one for which negative cycles may be introduced by magic sets.
Indeed, in this case the new rewriting only needs around 12\% of the execution time and around 16\% of the memory required by the previous rewriting on average.
Finally, concerning program $\Pi_3$, the previous rewriting could not be tested because it introduced recursive aggregates, as expected;
the new rewriting, instead, performed as for $\Pi_2$, with an average execution time of 4.5 seconds and an average memory consumption of 1.4 GB.

As for LUBM, its Datalog encoding consists of 132 rules and 83 predicate names, which become 382 rules and 216 predicate names after running magic sets as implemented in \textsc{i-dlv} 1.1.3.
Within \textsc{i-dlv} 1.1.4, instead, the magic sets rewriting comprises 210 rules and 118 predicate names.
In fact, several rules and predicate names are removed because of full-free adornments.
A few additional rules, precisely 17, are filtered out by the subsumption checks.
Within this respect, it is interesting to observe that the number of subsumption checks to perform without the hashing technique presented in Section~\ref{sec:subsumption} is 37\,600, in contrast to a significantly smaller number of 1\,159 checks actually performed;
the hashing technique reduced by around 97\% the number of subsumption checks.
Finally, concerning execution time, both versions scale almost linearly, with a slight advantage of the new magic sets:
\textsc{i-dlv} 1.1.3 reported an average execution time of around 529 seconds, with a minimum of around 42 seconds and a maximum of around 1060;
\textsc{i-dlv} 1.1.4 reported an average execution time of around 502 seconds, with a minimum of around 40 seconds and a maximum of around 1027.

%We performed an additional test and disabled the hashing technique, so that all pairs of rules were actually checked for subsumption. We observed an overhead of 0.02 seconds (instead of 0.001 seconds that is the time consumed when the hashing technique is enabled).

\section{Related work}\label{sec:rw}

Magic sets were originally introduced for Datalog programs \cite{DBLP:conf/pods/BancilhonMSU86}, and applied among other contexts to bottom-up analysis of logic programs \cite{DBLP:journals/jlp/CodishD95} and BDD-Based Deductive DataBases \cite{DBLP:conf/aplas/WhaleyACL05}.
Extending the technique to Datalog programs with stratified negation was nontrivial, as the perfect model semantics is not applicable to the rewritten program if recursive negation is introduced by magic sets.
Several semantics were considered in the literature to overcome the limitation of perfect model semantics.
Among them, some authors defined ad-hoc semantics for rewritten programs \cite{DBLP:conf/fgcs/KerisitP88,DBLP:journals/jlp/BalbinPRM91,DBLP:conf/pods/Behrend03}, while \citeANP{DBLP:journals/tcs/KempSS95}~\citeyear{DBLP:journals/tcs/KempSS95} and \citeANP{DBLP:journals/jacm/Ross94}~\citeyear{DBLP:journals/jacm/Ross94} considered well-founded semantics, and showed that the well-founded model of any rewritten program obtained from a Datalog program with stratified negation is two-valued.

A similar semantic issue arises for aggregations \cite{DBLP:conf/vldb/MumickPR90,DBLP:journals/is/FurfaroGGZ02}, as there is no general consensus for recursive aggregates \cite{DBLP:journals/ki/AlvianoF18}.
This fact explains why \textsc{dlv} did not support (dynamic) magic sets \cite{DBLP:journals/ai/AlvianoFGL12} for programs with aggregates, even if their correctness was shown also for programs with some form of aggregation \cite{DBLP:conf/lpnmr/AlvianoGL11}.
In fact, even if techniques to process programs with recursive aggregates are known \cite{DBLP:conf/iclp/GebserKKS09,DBLP:journals/tplp/AlvianoFG15,DBLP:conf/ijcai/AlvianoFG16}, they are in general less efficient than those for stratified aggregates;
for example, \emph{shared aggregate sets} \cite{DBLP:journals/tplp/AlvianoDM18} are currently implemented in \textsc{wasp} \cite{DBLP:conf/cilc/DodaroAFLRS11,DBLP:conf/lpnmr/AlvianoADLMR19} only in the stratified case.

Magic sets were applied to other extensions of Datalog, in particular to disjunctive Datalog under stable model semantics \cite{DBLP:journals/tkde/Greco03,DBLP:journals/tplp/GrecoGTZ05}.
For disjunctive Datalog, dynamic magic sets push the optimization on all phases of the computation of stable models \cite{DBLP:journals/ai/AlvianoFGL12}, and are shown to be correct for a semantic class known as \emph{super-coherent} programs \cite{DBLP:journals/aicom/AlvianoF11,DBLP:journals/tplp/AlvianoFW14}.
The restriction on SIPS applied in Section~\ref{sec:cycles} necessarily limits the optimization of dynamic magic sets to the grounding phase, which is anyhow the only computation phase for the language considered in this paper.
On the other hand, the restriction on SIPS presented in this paper does not inhibit the application of magic sets to programs characterized by multiple stable models:
magic sets would still optimize the grounding of those programs, so that other highly optimized techniques for computing \emph{cautious consequences} of propositional programs can be employed \cite{DBLP:journals/tplp/AlvianoDR14,DBLP:journals/tplp/AlvianoDJM18}, among them those based on \emph{unsatisfiable core analysis} \cite{DBLP:journals/tplp/AlvianoD16,DBLP:conf/ijcai/AlvianoD17}.

\section{Conclusion}\label{sec:conclusion}

Magic sets aim at optimizing query answering, but they may introduce recursive definitions that possibly deteriorate the performance of a bottom-up evaluation of the rewritten program.
Previous works in the literature noted the problem for programs with stratified negation, and proposed several solutions to the associated semantic issue.
By imposing some restriction on SIPS, this paper provides a simple solution to semantic issues arising for programs with stratified negation and aggregations, which also inhibits the creation of new positive recursive definitions (Section~\ref{sec:cycles}).
The role of magic atoms is to restrict the range of variables in the original rules of the processed program.
When all arguments of a predicate have to be considered free, a full-free adornment is generated.
Any other adornment associated with such a predicate only introduces overhead in the evaluation of the rewritten program.
This paper proposes a post-processing of the rewritten program to purge full-free adornments, in contrast to more complex unroll procedures (Section~\ref{sec:free}).
Further overhead is associated with subsumed rules.
Their identification is nontrivial and addressed by a backtracking algorithm.
Even if there are few branching points, actually only if there are multiple occurrences of the same predicate in rule bodies, running the backtracking algorithm for all pairs of rules in the rewritten program is expensive.
The hashing technique given in Section~\ref{sec:subsumption} provides a drastic reduction on the number of checks.

\section*{Acknowledgments}

This work has been partially supported by MIUR under project ``Declarative Reasoning over Streams'' (CUP H24I17000080001) -- PRIN 2017, by MISE under project ``S2BDW'' (F/050389/01-03/X32) -- ``Horizon2020'' PON I\&C2014-20, by Regione Calabria under project ``DLV LargeScale'' (CUP J28C17000220006) -- POR Calabria 2014-20, and by GNCS-INdAM.

\bibliographystyle{acmtrans}
\bibliography{bibtex}

\begin{thebibliography}{}

\bibitem[\protect\citeauthoryear{Adrian, Alviano, Calimeri, Cuteri, Dodaro,
  Faber, Fusc{\`{a}}, Leone, Manna, Perri, Ricca, Veltri, and Zangari}{Adrian
  et~al.}{2018}]{DBLP:journals/ki/AdrianACCDFFLMP18}
{\sc Adrian, W.~T.}, {\sc Alviano, M.}, {\sc Calimeri, F.}, {\sc Cuteri, B.},
  {\sc Dodaro, C.}, {\sc Faber, W.}, {\sc Fusc{\`{a}}, D.}, {\sc Leone, N.},
  {\sc Manna, M.}, {\sc Perri, S.}, {\sc Ricca, F.}, {\sc Veltri, P.}, {\sc
  and} {\sc Zangari, J.} 2018.
\newblock The {ASP} system {DLV:} advancements and applications.
\newblock {\em {KI}\/}~{\em 32,\/}~2-3, 177--179.

\bibitem[\protect\citeauthoryear{Alviano, Amendola, Dodaro, Leone, Maratea, and
  Ricca}{Alviano et~al.}{2019}]{DBLP:conf/lpnmr/AlvianoADLMR19}
{\sc Alviano, M.}, {\sc Amendola, G.}, {\sc Dodaro, C.}, {\sc Leone, N.}, {\sc
  Maratea, M.}, {\sc and} {\sc Ricca, F.} 2019.
\newblock Evaluation of disjunctive programs in {WASP}.
\newblock In {\sc M.~Balduccini}, {\sc Y.~Lierler}, {\sc and} {\sc S.~Woltran}
  (Eds.), {\em Logic Programming and Nonmonotonic Reasoning - 15th
  International Conference, {LPNMR} 2019, Philadelphia, PA, USA, June 3-7,
  2019, Proceedings}, Volume 11481 of {\em Lecture Notes in Computer Science},
  pp.\  241--255. Springer.

\bibitem[\protect\citeauthoryear{Alviano, Calimeri, Dodaro, Fusc{\`{a}}, Leone,
  Perri, Ricca, Veltri, and Zangari}{Alviano
  et~al.}{2017}]{DBLP:conf/lpnmr/AlvianoCDFLPRVZ17}
{\sc Alviano, M.}, {\sc Calimeri, F.}, {\sc Dodaro, C.}, {\sc Fusc{\`{a}}, D.},
  {\sc Leone, N.}, {\sc Perri, S.}, {\sc Ricca, F.}, {\sc Veltri, P.}, {\sc
  and} {\sc Zangari, J.} 2017.
\newblock The {ASP} system {DLV2}.
\newblock In {\sc M.~Balduccini} {\sc and} {\sc T.~Janhunen} (Eds.), {\em Logic
  Programming and Nonmonotonic Reasoning - 14th International Conference,
  {LPNMR} 2017, Espoo, Finland, July 3-6, 2017, Proceedings}, Volume 10377 of
  {\em Lecture Notes in Computer Science}, pp.\  215--221. Springer.

\bibitem[\protect\citeauthoryear{Alviano and Dodaro}{Alviano and
  Dodaro}{2016}]{DBLP:journals/tplp/AlvianoD16}
{\sc Alviano, M.} {\sc and} {\sc Dodaro, C.} 2016.
\newblock Anytime answer set optimization via unsatisfiable core shrinking.
\newblock {\em Theory and Practice of Logic Programming\/}~{\em 16,\/}~5-6,
  533--551.

\bibitem[\protect\citeauthoryear{Alviano and Dodaro}{Alviano and
  Dodaro}{2017}]{DBLP:conf/ijcai/AlvianoD17}
{\sc Alviano, M.} {\sc and} {\sc Dodaro, C.} 2017.
\newblock Unsatisfiable core shrinking for anytime answer set optimization.
\newblock In {\sc C.~Sierra} (Ed.), {\em Proceedings of the Twenty-Sixth
  International Joint Conference on Artificial Intelligence, {IJCAI} 2017,
  Melbourne, Australia, August 19-25, 2017}, pp.\  4781--4785. ijcai.org.

\bibitem[\protect\citeauthoryear{Alviano, Dodaro, J{\"{a}}rvisalo, Maratea, and
  Previti}{Alviano et~al.}{2018}]{DBLP:journals/tplp/AlvianoDJM18}
{\sc Alviano, M.}, {\sc Dodaro, C.}, {\sc J{\"{a}}rvisalo, M.}, {\sc Maratea,
  M.}, {\sc and} {\sc Previti, A.} 2018.
\newblock Cautious reasoning in {ASP} via minimal models and unsatisfiable
  cores.
\newblock {\em Theory and Practice of Logic Programming\/}~{\em 18,\/}~3-4,
  319--336.

\bibitem[\protect\citeauthoryear{Alviano, Dodaro, and Maratea}{Alviano
  et~al.}{2018}]{DBLP:journals/tplp/AlvianoDM18}
{\sc Alviano, M.}, {\sc Dodaro, C.}, {\sc and} {\sc Maratea, M.} 2018.
\newblock Shared aggregate sets in answer set programming.
\newblock {\em Theory and Practice of Logic Programming\/}~{\em 18,\/}~3-4,
  301--318.

\bibitem[\protect\citeauthoryear{Alviano, Dodaro, and Ricca}{Alviano
  et~al.}{2014}]{DBLP:journals/tplp/AlvianoDR14}
{\sc Alviano, M.}, {\sc Dodaro, C.}, {\sc and} {\sc Ricca, F.} 2014.
\newblock Anytime computation of cautious consequences in answer set
  programming.
\newblock {\em Theory and Practice of Logic Programming\/}~{\em 14,\/}~4-5,
  755--770.

\bibitem[\protect\citeauthoryear{Alviano and Faber}{Alviano and
  Faber}{2011}]{DBLP:journals/aicom/AlvianoF11}
{\sc Alviano, M.} {\sc and} {\sc Faber, W.} 2011.
\newblock Dynamic magic sets and super-coherent answer set programs.
\newblock {\em {AI} Commun.\/}~{\em 24,\/}~2, 125--145.

\bibitem[\protect\citeauthoryear{Alviano and Faber}{Alviano and
  Faber}{2018}]{DBLP:journals/ki/AlvianoF18}
{\sc Alviano, M.} {\sc and} {\sc Faber, W.} 2018.
\newblock Aggregates in answer set programming.
\newblock {\em {KI}\/}~{\em 32,\/}~2-3, 119--124.

\bibitem[\protect\citeauthoryear{Alviano, Faber, and Gebser}{Alviano
  et~al.}{2015}]{DBLP:journals/tplp/AlvianoFG15}
{\sc Alviano, M.}, {\sc Faber, W.}, {\sc and} {\sc Gebser, M.} 2015.
\newblock Rewriting recursive aggregates in answer set programming: back to
  monotonicity.
\newblock {\em Theory and Practice of Logic Programming\/}~{\em 15,\/}~4-5,
  559--573.

\bibitem[\protect\citeauthoryear{Alviano, Faber, and Gebser}{Alviano
  et~al.}{2016}]{DBLP:conf/ijcai/AlvianoFG16}
{\sc Alviano, M.}, {\sc Faber, W.}, {\sc and} {\sc Gebser, M.} 2016.
\newblock From non-convex aggregates to monotone aggregates in {ASP}.
\newblock In {\sc S.~Kambhampati} (Ed.), {\em Proceedings of the Twenty-Fifth
  International Joint Conference on Artificial Intelligence, {IJCAI} 2016, New
  York, NY, USA, 9-15 July 2016}, pp.\  4100--4194. {IJCAI/AAAI} Press.

\bibitem[\protect\citeauthoryear{Alviano, Faber, Greco, and Leone}{Alviano
  et~al.}{2012}]{DBLP:journals/ai/AlvianoFGL12}
{\sc Alviano, M.}, {\sc Faber, W.}, {\sc Greco, G.}, {\sc and} {\sc Leone, N.}
  2012.
\newblock Magic sets for disjunctive datalog programs.
\newblock {\em Artif. Intell.\/}~{\em 187}, 156--192.

\bibitem[\protect\citeauthoryear{Alviano, Faber, and Woltran}{Alviano
  et~al.}{2014}]{DBLP:journals/tplp/AlvianoFW14}
{\sc Alviano, M.}, {\sc Faber, W.}, {\sc and} {\sc Woltran, S.} 2014.
\newblock Complexity of super-coherence problems in {ASP}.
\newblock {\em Theory and Practice of Logic Programming\/}~{\em 14,\/}~3,
  339--361.

\bibitem[\protect\citeauthoryear{Alviano, Greco, and Leone}{Alviano
  et~al.}{2011}]{DBLP:conf/lpnmr/AlvianoGL11}
{\sc Alviano, M.}, {\sc Greco, G.}, {\sc and} {\sc Leone, N.} 2011.
\newblock Dynamic magic sets for programs with monotone recursive aggregates.
\newblock In {\sc J.~P. Delgrande} {\sc and} {\sc W.~Faber} (Eds.), {\em Logic
  Programming and Nonmonotonic Reasoning - 11th International Conference,
  {LPNMR} 2011, Vancouver, Canada, May 16-19, 2011. Proceedings}, Volume 6645
  of {\em Lecture Notes in Computer Science}, pp.\  148--160. Springer.

\bibitem[\protect\citeauthoryear{Balbin, Port, Ramamohanarao, and
  Meenakshi}{Balbin et~al.}{1991}]{DBLP:journals/jlp/BalbinPRM91}
{\sc Balbin, I.}, {\sc Port, G.~S.}, {\sc Ramamohanarao, K.}, {\sc and} {\sc
  Meenakshi, K.} 1991.
\newblock Efficient bottom-up computation of queries on stratified databases.
\newblock {\em J. Log. Program.\/}~{\em 11,\/}~3{\&}4, 295--344.

\bibitem[\protect\citeauthoryear{Bancilhon, Maier, Sagiv, and Ullman}{Bancilhon
  et~al.}{1986}]{DBLP:conf/pods/BancilhonMSU86}
{\sc Bancilhon, F.}, {\sc Maier, D.}, {\sc Sagiv, Y.}, {\sc and} {\sc Ullman,
  J.~D.} 1986.
\newblock Magic sets and other strange ways to implement logic programs.
\newblock In {\sc A.~Silberschatz} (Ed.), {\em Proceedings of the Fifth {ACM}
  {SIGACT-SIGMOD} Symposium on Principles of Database Systems, March 24-26,
  1986, Cambridge, Massachusetts, {USA}}, pp.\  1--15. {ACM}.

\bibitem[\protect\citeauthoryear{Bartholomew, Lee, and Meng}{Bartholomew
  et~al.}{2011}]{DBLP:conf/aaaiss/BartholomewLM11}
{\sc Bartholomew, M.}, {\sc Lee, J.}, {\sc and} {\sc Meng, Y.} 2011.
\newblock First-order semantics of aggregates in answer set programming via
  modified circumscription.
\newblock In {\em Logical Formalizations of Commonsense Reasoning, Papers from
  the 2011 {AAAI} Spring Symposium, Technical Report SS-11-06, Stanford,
  California, USA, March 21-23, 2011}. {AAAI}.

\bibitem[\protect\citeauthoryear{Beeri and Ramakrishnan}{Beeri and
  Ramakrishnan}{1991}]{DBLP:journals/jlp/BeeriR91}
{\sc Beeri, C.} {\sc and} {\sc Ramakrishnan, R.} 1991.
\newblock On the power of magic.
\newblock {\em J. Log. Program.\/}~{\em 10,\/}~3{\&}4, 255--299.

\bibitem[\protect\citeauthoryear{Behrend}{Behrend}{2003}]{DBLP:conf/pods/Behrend03}
{\sc Behrend, A.} 2003.
\newblock Soft stratification for magic set based query evaluation in deductive
  databases.
\newblock In {\sc F.~Neven}, {\sc C.~Beeri}, {\sc and} {\sc T.~Milo} (Eds.),
  {\em Proceedings of the Twenty-Second {ACM} {SIGACT-SIGMOD-SIGART} Symposium
  on Principles of Database Systems, June 9-12, 2003, San Diego, CA, {USA}},
  pp.\  102--110. {ACM}.

\bibitem[\protect\citeauthoryear{Codish and Demoen}{Codish and
  Demoen}{1995}]{DBLP:journals/jlp/CodishD95}
{\sc Codish, M.} {\sc and} {\sc Demoen, B.} 1995.
\newblock Analyzing logic programs using "\texttt{PROP}"-ositional logic
  programs and a magic wand.
\newblock {\em J. Log. Program.\/}~{\em 25,\/}~3, 249--274.

\bibitem[\protect\citeauthoryear{Dodaro, Alviano, Faber, Leone, Ricca, and
  Sirianni}{Dodaro et~al.}{2011}]{DBLP:conf/cilc/DodaroAFLRS11}
{\sc Dodaro, C.}, {\sc Alviano, M.}, {\sc Faber, W.}, {\sc Leone, N.}, {\sc
  Ricca, F.}, {\sc and} {\sc Sirianni, M.} 2011.
\newblock The birth of a {WASP:} preliminary report on a new {ASP} solver.
\newblock In {\sc F.~Fioravanti} (Ed.), {\em Proceedings of the 26th Italian
  Conference on Computational Logic, Pescara, Italy, August 31 - September 2,
  2011}, Volume 810 of {\em {CEUR} Workshop Proceedings}, pp.\  99--113.
  CEUR-WS.org.

\bibitem[\protect\citeauthoryear{Eiter, Ortiz, Simkus, Tran, and Xiao}{Eiter
  et~al.}{2012}]{DBLP:conf/aaai/EiterOSTX12}
{\sc Eiter, T.}, {\sc Ortiz, M.}, {\sc Simkus, M.}, {\sc Tran, T.}, {\sc and}
  {\sc Xiao, G.} 2012.
\newblock Query rewriting for horn-shiq plus rules.
\newblock In {\sc J.~Hoffmann} {\sc and} {\sc B.~Selman} (Eds.), {\em
  Proceedings of the Twenty-Sixth {AAAI} Conference on Artificial Intelligence,
  July 22-26, 2012, Toronto, Ontario, Canada.} {AAAI} Press.

\bibitem[\protect\citeauthoryear{Faber, Pfeifer, and Leone}{Faber
  et~al.}{2011}]{DBLP:journals/ai/FaberPL11}
{\sc Faber, W.}, {\sc Pfeifer, G.}, {\sc and} {\sc Leone, N.} 2011.
\newblock Semantics and complexity of recursive aggregates in answer set
  programming.
\newblock {\em Artif. Intell.\/}~{\em 175,\/}~1, 278--298.

\bibitem[\protect\citeauthoryear{Ferraris}{Ferraris}{2011}]{DBLP:journals/tocl/Ferraris11}
{\sc Ferraris, P.} 2011.
\newblock Logic programs with propositional connectives and aggregates.
\newblock {\em {ACM} Trans. Comput. Log.\/}~{\em 12,\/}~4, 25.

\bibitem[\protect\citeauthoryear{Furfaro, Greco, Ganguly, and Zaniolo}{Furfaro
  et~al.}{2002}]{DBLP:journals/is/FurfaroGGZ02}
{\sc Furfaro, F.}, {\sc Greco, S.}, {\sc Ganguly, S.}, {\sc and} {\sc Zaniolo,
  C.} 2002.
\newblock Pushing extrema aggregates to optimize logic queries.
\newblock {\em Inf. Syst.\/}~{\em 27,\/}~5, 321--343.

\bibitem[\protect\citeauthoryear{Gebser, Kaminski, Kaufmann, and Schaub}{Gebser
  et~al.}{2009}]{DBLP:conf/iclp/GebserKKS09}
{\sc Gebser, M.}, {\sc Kaminski, R.}, {\sc Kaufmann, B.}, {\sc and} {\sc
  Schaub, T.} 2009.
\newblock On the implementation of weight constraint rules in conflict-driven
  {ASP} solvers.
\newblock In {\sc P.~M. Hill} {\sc and} {\sc D.~S. Warren} (Eds.), {\em Logic
  Programming, 25th International Conference, {ICLP} 2009, Pasadena, CA, USA,
  July 14-17, 2009. Proceedings}, Volume 5649 of {\em Lecture Notes in Computer
  Science}, pp.\  250--264. Springer.

\bibitem[\protect\citeauthoryear{Gelder}{Gelder}{1989}]{DBLP:journals/jlp/Gelder89}
{\sc Gelder, A.~V.} 1989.
\newblock Negation as failure using tight derivations for general logic
  programs.
\newblock {\em J. Log. Program.\/}~{\em 6,\/}~1{\&}2, 109--133.

\bibitem[\protect\citeauthoryear{Gelder, Ross, and Schlipf}{Gelder
  et~al.}{1991}]{DBLP:journals/jacm/GelderRS91}
{\sc Gelder, A.~V.}, {\sc Ross, K.~A.}, {\sc and} {\sc Schlipf, J.~S.} 1991.
\newblock The well-founded semantics for general logic programs.
\newblock {\em J. {ACM}\/}~{\em 38,\/}~3, 620--650.

\bibitem[\protect\citeauthoryear{Gelfond and Lifschitz}{Gelfond and
  Lifschitz}{1991}]{DBLP:journals/ngc/GelfondL91}
{\sc Gelfond, M.} {\sc and} {\sc Lifschitz, V.} 1991.
\newblock Classical negation in logic programs and disjunctive databases.
\newblock {\em New Generation Comput.\/}~{\em 9,\/}~3/4, 365--386.

\bibitem[\protect\citeauthoryear{Gelfond and Zhang}{Gelfond and
  Zhang}{2014}]{DBLP:journals/tplp/GelfondZ14}
{\sc Gelfond, M.} {\sc and} {\sc Zhang, Y.} 2014.
\newblock Vicious circle principle and logic programs with aggregates.
\newblock {\em Theory and Practice of Logic Programming\/}~{\em 14,\/}~4-5,
  587--601.

\bibitem[\protect\citeauthoryear{Greco, Greco, Trubitsyna, and Zumpano}{Greco
  et~al.}{2005}]{DBLP:journals/tplp/GrecoGTZ05}
{\sc Greco, G.}, {\sc Greco, S.}, {\sc Trubitsyna, I.}, {\sc and} {\sc Zumpano,
  E.} 2005.
\newblock Optimization of bound disjunctive queries with constraints.
\newblock {\em Theory and Practice of Logic Programming\/}~{\em 5,\/}~6,
  713--745.

\bibitem[\protect\citeauthoryear{Greco}{Greco}{2003}]{DBLP:journals/tkde/Greco03}
{\sc Greco, S.} 2003.
\newblock Binding propagation techniques for the optimization of bound
  disjunctive queries.
\newblock {\em {IEEE} Trans. Knowl. Data Eng.\/}~{\em 15,\/}~2, 368--385.

\bibitem[\protect\citeauthoryear{Kemp, Srivastava, and Stuckey}{Kemp
  et~al.}{1995}]{DBLP:journals/tcs/KempSS95}
{\sc Kemp, D.~B.}, {\sc Srivastava, D.}, {\sc and} {\sc Stuckey, P.~J.} 1995.
\newblock Bottom-up evaluation and query optimization of well-founded models.
\newblock {\em Theor. Comput. Sci.\/}~{\em 146,\/}~1{\&}2, 145--184.

\bibitem[\protect\citeauthoryear{Kerisit and Pugin}{Kerisit and
  Pugin}{1988}]{DBLP:conf/fgcs/KerisitP88}
{\sc Kerisit, J.} {\sc and} {\sc Pugin, J.} 1988.
\newblock Efficient query answering on stratified databases.
\newblock In {\em {FGCS}}, pp.\  719--726.

\bibitem[\protect\citeauthoryear{Leone, Allocca, Alviano, Calimeri, Civili,
  Costabile, Cuteri, Fiorentino, Fusc{\`{a}}, Germano, Laboccetta, Manna,
  Perri, Reale, Ricca, Veltri, and Zangari}{Leone
  et~al.}{2019}]{DBLP:conf/cilc/LeoneAACCCCFFGL19}
{\sc Leone, N.}, {\sc Allocca, C.}, {\sc Alviano, M.}, {\sc Calimeri, F.}, {\sc
  Civili, C.}, {\sc Costabile, R.}, {\sc Cuteri, B.}, {\sc Fiorentino, A.},
  {\sc Fusc{\`{a}}, D.}, {\sc Germano, S.}, {\sc Laboccetta, G.}, {\sc Manna,
  M.}, {\sc Perri, S.}, {\sc Reale, K.}, {\sc Ricca, F.}, {\sc Veltri, P.},
  {\sc and} {\sc Zangari, J.} 2019.
\newblock Large scale {DLV:} preliminary results.
\newblock In {\sc A.~Casagrande} {\sc and} {\sc E.~G. Omodeo} (Eds.), {\em
  Proceedings of the 34th Italian Conference on Computational Logic, Trieste,
  Italy, June 19-21, 2019.}, Volume 2396 of {\em {CEUR} Workshop Proceedings}.
  CEUR-WS.org.

\bibitem[\protect\citeauthoryear{Leone, Allocca, Alviano, Calimeri, Civili,
  Costabile, Fiorentino, Fusc{\`{a}}, Germano, Laboccetta, Cuteri, Manna,
  Perri, Reale, Ricca, Veltri, and Zangari}{Leone
  et~al.}{2019}]{DBLP:conf/lpnmr/LeoneAACCCFFGLC19}
{\sc Leone, N.}, {\sc Allocca, C.}, {\sc Alviano, M.}, {\sc Calimeri, F.}, {\sc
  Civili, C.}, {\sc Costabile, R.}, {\sc Fiorentino, A.}, {\sc Fusc{\`{a}},
  D.}, {\sc Germano, S.}, {\sc Laboccetta, G.}, {\sc Cuteri, B.}, {\sc Manna,
  M.}, {\sc Perri, S.}, {\sc Reale, K.}, {\sc Ricca, F.}, {\sc Veltri, P.},
  {\sc and} {\sc Zangari, J.} 2019.
\newblock Enhancing {DLV} for large-scale reasoning.
\newblock In {\sc M.~Balduccini}, {\sc Y.~Lierler}, {\sc and} {\sc S.~Woltran}
  (Eds.), {\em Logic Programming and Nonmonotonic Reasoning - 15th
  International Conference, {LPNMR} 2019, Philadelphia, PA, USA, June 3-7,
  2019, Proceedings}, Volume 11481 of {\em Lecture Notes in Computer Science},
  pp.\  312--325. Springer.

\bibitem[\protect\citeauthoryear{Liu, Pontelli, Son, and Truszczynski}{Liu
  et~al.}{2010}]{DBLP:journals/ai/LiuPST10}
{\sc Liu, L.}, {\sc Pontelli, E.}, {\sc Son, T.~C.}, {\sc and} {\sc
  Truszczynski, M.} 2010.
\newblock Logic programs with abstract constraint atoms: The role of
  computations.
\newblock {\em Artif. Intell.\/}~{\em 174,\/}~3-4, 295--315.

\bibitem[\protect\citeauthoryear{Mumick, Pirahesh, and Ramakrishnan}{Mumick
  et~al.}{1990}]{DBLP:conf/vldb/MumickPR90}
{\sc Mumick, I.~S.}, {\sc Pirahesh, H.}, {\sc and} {\sc Ramakrishnan, R.} 1990.
\newblock The magic of duplicates and aggregates.
\newblock In {\sc D.~McLeod}, {\sc R.~Sacks{-}Davis}, {\sc and} {\sc H.~Schek}
  (Eds.), {\em 16th International Conference on Very Large Data Bases, August
  13-16, 1990, Brisbane, Queensland, Australia, Proceedings.}, pp.\  264--277.
  Morgan Kaufmann.

\bibitem[\protect\citeauthoryear{Pelov, Denecker, and Bruynooghe}{Pelov
  et~al.}{2007}]{DBLP:journals/tplp/PelovDB07}
{\sc Pelov, N.}, {\sc Denecker, M.}, {\sc and} {\sc Bruynooghe, M.} 2007.
\newblock Well-founded and stable semantics of logic programs with aggregates.
\newblock {\em Theory and Practice of Logic Programming\/}~{\em 7,\/}~3,
  301--353.

\bibitem[\protect\citeauthoryear{Przymusinski}{Przymusinski}{1989}]{DBLP:journals/jar/Przymusinski89}
{\sc Przymusinski, T.~C.} 1989.
\newblock On the declarative and procedural semantics of logic programs.
\newblock {\em J. Autom. Reasoning\/}~{\em 5,\/}~2, 167--205.

\bibitem[\protect\citeauthoryear{Ross}{Ross}{1994}]{DBLP:journals/jacm/Ross94}
{\sc Ross, K.~A.} 1994.
\newblock Modular stratification and magic sets for datalog programs with
  negation.
\newblock {\em J. {ACM}\/}~{\em 41,\/}~6, 1216--1266.

\bibitem[\protect\citeauthoryear{Simons, Niemel{\"{a}}, and Soininen}{Simons
  et~al.}{2002}]{DBLP:journals/ai/SimonsNS02}
{\sc Simons, P.}, {\sc Niemel{\"{a}}, I.}, {\sc and} {\sc Soininen, T.} 2002.
\newblock Extending and implementing the stable model semantics.
\newblock {\em Artif. Intell.\/}~{\em 138,\/}~1-2, 181--234.

\bibitem[\protect\citeauthoryear{Stuckey and Sudarshan}{Stuckey and
  Sudarshan}{1994}]{DBLP:conf/pods/StuckeyS94}
{\sc Stuckey, P.~J.} {\sc and} {\sc Sudarshan, S.} 1994.
\newblock Compiling query constraints.
\newblock In {\sc V.~Vianu} (Ed.), {\em Proceedings of the Thirteenth {ACM}
  {SIGACT-SIGMOD-SIGART} Symposium on Principles of Database Systems, May
  24-26, 1994, Minneapolis, Minnesota, {USA}}, pp.\  56--67. {ACM} Press.

\bibitem[\protect\citeauthoryear{Whaley, Avots, Carbin, and Lam}{Whaley
  et~al.}{2005}]{DBLP:conf/aplas/WhaleyACL05}
{\sc Whaley, J.}, {\sc Avots, D.}, {\sc Carbin, M.}, {\sc and} {\sc Lam, M.~S.}
  2005.
\newblock Using datalog with binary decision diagrams for program analysis.
\newblock In {\sc K.~Yi} (Ed.), {\em Programming Languages and Systems, Third
  Asian Symposium, {APLAS} 2005, Tsukuba, Japan, November 2-5, 2005,
  Proceedings}, Volume 3780 of {\em Lecture Notes in Computer Science}, pp.\
  97--118. Springer.

\end{thebibliography}

\label{lastpage}
\end{document}